\documentclass[conference,letterpaper]{IEEEtran}

\usepackage[utf8]{inputenc} 
\usepackage[T1]{fontenc}
\usepackage{url}
\usepackage{ifthen}
\usepackage[cmex10]{amsmath} 
\usepackage{cite}
\usepackage{bbm}
\usepackage{color}
\usepackage{comment}
\usepackage{amssymb}
\usepackage{amsmath}
\usepackage{amsthm}
\usepackage [english]{babel}
\usepackage [autostyle, english = american]{csquotes}
\usepackage{mathtools}
\usepackage[hang,flushmargin]{footmisc}

\MakeOuterQuote{"}

\usepackage[short]{optidef}

\usepackage{stfloats}
\newtheorem{definition}{Definition}

\newtheorem{theorem}{Theorem}
\newtheorem{corollary}{Corollary}

\newtheorem{prop}{Proposition}

\newcommand\blfootnote[1]{%
  \begingroup
  \renewcommand\thefootnote{}\footnote{#1}%
  \addtocounter{footnote}{-1}%
  \endgroup
}

\begin{document}
\title{A Tunable Loss Function for Binary Classification} 


\author{%
   \IEEEauthorblockN{Tyler Sypherd\IEEEauthorrefmark{1},
                    Mario Diaz\IEEEauthorrefmark{1}\IEEEauthorrefmark{2},
                    Lalitha Sankar\IEEEauthorrefmark{1},
                    and Peter Kairouz\IEEEauthorrefmark{3}}
   \IEEEauthorblockA{\IEEEauthorrefmark{1}%
                     Arizona State University, \{tsypherd,mdiaztor,lsankar\}@asu.edu}
   \IEEEauthorblockA{\IEEEauthorrefmark{2}%
                     Centro de Investigaci\'on en Matem\'aticas A.C.,
                     diaztorres@cimat.mx}
   \IEEEauthorblockA{\IEEEauthorrefmark{3}%
                     Google AI,\ kairouz@google.com}
}
\maketitle

\begin{abstract}
We present $\alpha$-loss, $\alpha \in [1,\infty]$, a tunable loss function for binary classification that bridges log-loss ($\alpha=1$) and $0$-$1$ loss ($\alpha = \infty$). We prove that $\alpha$-loss has an equivalent margin-based form and is classification-calibrated, two desirable properties for a good surrogate loss function for the ideal yet intractable $0$-$1$ loss.
For logistic regression-based classification, we provide an upper bound on the difference between the empirical and expected risk for $\alpha$-loss at the critical points of the empirical risk by exploiting its Lipschitzianity along with recent results on the landscape features of empirical risk functions.
Finally, we show that $\alpha$-loss with $\alpha = 2$ performs better than log-loss on MNIST for logistic regression.


\blfootnote{This material is based upon work supported by the National Science Foundation under Grant Nos. CCF-1350914 and CIF-1815261.}
\end{abstract}

\section{Introduction}

In learning theory, the performance of a classification algorithm in terms of accuracy, tractability, and convergence guarantees is contingent on the choice of a loss function. 
%
Consider a feature vector $X \in \mathcal{X}$, an unknown finite label $Y \in \mathcal{Y}$, and a hypothesis test $h:\mathcal{X} \rightarrow \mathcal{Y}$. The canonical $0$-$1$ loss, given by $\mathbbm{1}[h(X) \neq Y]$, is considered an ideal loss function that captures the probability of incorrectly guessing the true label $Y$ using $h(X)$. However, since the $0$-$1$ loss is neither continuous nor differentiable, its practical application is intractable with state-of-the-art learning algorithms. As a result, there has been much interest in identifying surrogate loss functions that best approximate the $0$-$1$ loss. Common surrogate loss functions include logistic loss, squared loss, and hinge loss. 
\par
{For binary classification tasks, a hypothesis test $h:\mathcal{X}\to\{-1,1\}$ is typically replaced by a classification function $f:\mathcal{X} \rightarrow \overline{\mathbb{R}}$, where $\overline{\mathbb{R}} = \mathbb{R} \cup \{\pm{\infty}\}$. In this context, loss functions are often written in terms of a \textit{margin}, defined as the product of the label, $Y \in \{-1,1\}$, and the value of the classification function $f(X)$ (see,  \cite{lin2004note,bartlett2006convexity,nguyen2009,masnadi2009design}).}
In \cite{lin2004note}, Lin defines a margin-based loss function as Fisher consistent if, for any $x$ and a given posterior $P_{Y|X=x}$, its population minimizer has the same sign as the optimal Bayes classifier.
In \cite{bartlett2006convexity}, Bartlett \textit{et al.}~introduce a stronger surrogate requirement of \textit{classification-calibration} wherein the loss function is Fisher consistent for any $P_{Y|X=x}$. 

Yet another property for a good surrogate loss function is captured by the effectiveness of the empirical risk minimizers in approximating the true risk minimizers, a property studied through the \textit{empirical landscape}.
In \cite{mei2018landscape}, Mei \textit{et al.} prove that for general non-convex loss functions which satisfy certain regularity conditions, all critical features of the landscape including local minimizers/maximizers and saddle points of the empirical risk and the true risk are one-to-one, with the distance between corresponding features decreasing as $O\left(\sqrt{\log{n}/n}\right)$ for $n$ samples. 
\par In \cite{liao2018tunable}, Liao \textit{et al.} introduce $\alpha$-loss as a new loss function to model information leakage under different adversarial threat models. 
We consider a more general learning setting and apply $\alpha$-loss for binary classification. We prove that $\alpha$-loss has an equivalent margin-based form which is classification-calibrated. For a family of logistic regression based classifiers, we use the Lipschitzianity of $\alpha$-loss and results in \cite{mei2018landscape} to upper bound the difference between the empirical and expected risk under $\alpha$-loss at the critical points of the empirical risk. 
Finally, for the MNIST dataset, we focus on a low capacity learning model using logistic regression (such models are desirable when tuning deep neural networks is challenging) to illustrate the higher classification accuracy of $\alpha$-loss ($\alpha >1$) relative to the oft-used cross entropy (log-loss).

\section{Preliminaries}

\subsection{$\alpha$-loss}
%
Let $\mathcal{P}(\mathcal{Y})$ be the set of probability distributions over $\mathcal{Y}$. For $\alpha \in [1,\infty]$, Liao et al. \cite{liao2018tunable} define $\alpha$-loss $l^{\alpha}:\mathcal{Y} \times \mathcal{P}(\mathcal{Y}) \rightarrow \mathbb{R}_{+}$ as
\begin{equation} \label{def1}
l^{\alpha}(y,P_{Y}) := \begin{cases} 
      -\log{P_{Y}(y)} & \alpha = 1, \\
      \frac{\alpha}{\alpha - 1}[1 - P_{Y}(y)^{1 - 1/\alpha}] & \alpha \in (1,\infty), \\
      1 - P_{Y}(y) & \alpha = \infty.
\end{cases}
\end{equation} 
Note that for $(y,P_{Y})$ fixed, $l^{\alpha}(y,P_{Y})$ is continuous in $\alpha$. 

Consider random variables $(X,Y) \sim P_{X,Y}$. Observing $X$, one can construct an estimate $\hat{Y}$ of $Y$ such that $Y - X - \hat{Y}$ form a Markov chain. One can use expected $\alpha$-loss to quantify the effectiveness of the estimated posterior $P_{\hat{Y}|X}$ as $\mathbb{E}_{X,Y}[l^{\alpha}(Y,P_{\hat{Y}|X})]$. In particular, 
\begin{equation} \label{inf1}
\mathbb{E}_{X,Y}\left[l^{1}(Y,P_{\hat{Y}|X})\right] = \mathbb{E}_{X}\left[H(P_{Y|X=x},P_{\hat{Y}|X=x})\right],
\end{equation}
where $H(P,Q) := H(P) + D_{\textnormal{KL}}(P\|Q)$ is the cross-entropy between $P$ and $Q$. Similarly,
\begin{equation} \label{inf2}
\mathbb{E}_{X,Y}[l^{\infty}(Y,P_{\hat{Y}|X})] = \mathbb{P}[Y \neq \hat{Y}],
\end{equation}
i.e., the expected $\alpha$-loss for $\alpha = \infty$ equals the probability of error. It can be shown that the expected $\alpha$-loss is continuous in $\alpha$, i.e., \eqref{inf1} and \eqref{inf2} result from the continuous extensions for $\alpha=1$ and $\alpha=\infty$, respectively. Thus, we see that the extremal points of expected $\alpha$-loss are expected log-loss and probability of error. 

{
}

\subsection{Binary Classification in Learning}

Let $S_{n} = \{(X_{i},Y_{i}) : i = 1,\ldots,n\}$ be a training dataset where, for each $i$, $X_{i} \in \mathcal{X} \subset \mathbb{R}^{d}$ is the feature vector and $Y_{i} \in \mathcal{Y} = \{-1,1\}$ is the class label. We assume that the samples $\{(X_{i},Y_{i}) : i=1,\ldots,n\}$ are independently drawn from an unknown distribution $P_{X,Y}$. There are multiple approaches (and nomenclatures) to classification \cite{lin2004note,bartlett2006convexity,nguyen2009,masnadi2009design}; in particular, we consider two alternative approaches, namely, using \textit{soft classifiers} and using \textit{classification functions}. 

\textbf{Soft classifier}: In this approach, the objective of the learner is to construct, based on the training dataset $S_n$, a soft classifier $g:\mathcal{X}\to[0,1]$ capable of predicting the likelihood of a label of previously unseen feature vectors. More specifically, for each $x\in\mathcal{X}$, $g(x)$ estimates the probability of the event $\{Y=1\}$ given $\{X=x\}$. Usually, the learner selects a soft classifier by minimizing a loss function over a family of soft classifiers. Note that every soft classifier determines a set of beliefs and vice versa. Indeed, given a soft classifier $g$, we can define $P_{\hat{Y}|X}$ by taking $P_{\hat{Y}|X}(1|x) := g(x)$. Conversely, given a set of beliefs $P_{\hat{Y}|X}$, we can define a soft classifier $g(x) = P_{\hat{Y}|X}(1|x)$. 

Observe that the soft classification construct defined above makes $\alpha$-loss in \eqref{def1} a natural fit as a loss function. Indeed, one can define the expected $\alpha$-loss (true risk) of a soft classifier as
\begin{equation}
R_{l^{\alpha}}(g) = \mathbb{E}_{X,Y}[l^{\alpha}(Y, P_{\hat{Y}|X})],
\end{equation}
where $P_{\hat{Y}|X}$ is the set of beliefs associated to $g$. Analogously, we define the empirical $\alpha$-loss as 
\begin{equation}
    \hat{R}_{l^{\alpha}}(g) = \frac{1}{n} \sum_{i=1}^{n} l^{\alpha}(y_{i},P_{\hat{Y}|X=x_{i}}).
\end{equation}
Finally, we denote the conditional risk of the $\alpha$-loss by
\begin{equation}
\label{eq:DefConditionalRisk}
C_{l^{\alpha}}(g) = \mathbb{E}_{Y|X}[l^{\alpha}(Y,P_{\hat{Y}|X=x})].
\end{equation}
Observe that $R_{l^{\alpha}}(g) = \mathbb{E}_{X}[C_{l^{\alpha}}(g)]$.

\textbf{Classification function}: As an alternative approach, a learner can select a classification function $f:\mathcal{X}\rightarrow \overline{\mathbb{R}}$ by minimizing a loss function over a given family of classification functions. 
Observe that any such $f$ can yield a (hard decision) hypothesis $h(X) = \text{sign}(f(X))$. The value $f(x)$ can be regarded as the confidence on the value of $Y$ given $\{X=x\}$; a large value of $f(x)$ corresponds to a high confidence on the event $\{Y=1\}$ given $\{X=x\}$, while a large value of $-f(x)$ corresponds to a high confidence on the event $\{Y=-1\}$.

For this setting, margin-based loss functions have been proposed as a meaningful family of loss functions. A loss function is said to be margin-based if, for all $x\in\mathcal{X}$ and $y\in\mathcal{Y}$, the risk associated to a pair $(y,f(x))$ is given by $\tilde{l}(yf(x))$ for some function $\tilde{l}: \overline{\mathbb{R}}\to\mathbb{R}_{+}$. In this case, the risk of the pair $(y,f(x))$ only depends on the product $yf(x)$, where the product $yf(x)$ is called the margin. Observe that a negative margin corresponds to a mismatch between the signs of $f(x)$ and $y$, i.e., a classification error by $f$. Similarly, a positive margin corresponds to a match between the signs of $f(x)$ and $y$, i.e., a correct classification by $f$. Hence, most margin-based losses have a graph similar to those depicted in Figure~\ref{Fig:MarginBasedLosses}(a). Since margin-based loss functions synthesize two quantities ($Y$ and $f$) into a single margin, they are commonly found in the binary classification literature \cite{bartlett2006convexity,lin2004note,janocha2017loss}. The risk of a classification function $f$ with respect to (w.r.t.)~a margin-based loss function $\tilde{l}$ is defined as
\begin{equation}
R_{\tilde{l}}(f) = \mathbb{E}_{X,Y}[\tilde{l}(Yf(X))].
\end{equation}
For notational convenience, the risk of the $0$-$1$ loss is denoted by $R(f)$, i.e.,
\begin{equation}
R(f) = \mathbb{E}[\mathbbm{1}(\text{sign}(f(X)) \neq Y)].
\end{equation}

We now introduce a margin-based $\alpha$-loss. Let $\sigma:\overline{\mathbb{R}}\to[0,1]$ be the sigmoid function, i.e.,
\begin{equation}
\label{eq:DefSigmoid}
    \sigma(z) = \frac{1}{1+e^{-z}},
\end{equation}
Observe that $\sigma$ is invertible and $\sigma^{-1}:[0,1]\to\overline{\mathbb{R}}$ is given by
\begin{equation}
    \sigma^{-1}(z) = \log\left(\frac{z}{1-z}\right).
\end{equation}

\begin{definition} 
We define the margin $\alpha$-loss $\tilde{l}^{\alpha}: \overline{\mathbb{R}} \to \mathbb{R}_{+}$ as
\begin{equation}
\label{mar}
\tilde{l}^{\alpha}(z) := \begin{cases} 
      -\log(\sigma(z)) & \alpha = 1, \\
      \frac{\alpha}{\alpha - 1}\left(1 - \sigma(z)^{1 - 1/\alpha}\right) & \alpha \in (1,\infty), \\
      1 - \sigma(z) & \alpha = \infty.
\end{cases}
\end{equation}
\end{definition}

In Figure~\ref{Fig:MarginBasedLosses}(a), we plot the margin-based $\alpha$-loss for different values of $\alpha$. Observe that, on the one hand, the penalty assigned to misclassified examples decreases as $\alpha$ increases. In practice, this decrease is desirable as the classification error only depends on the prediction itself and not in the particular confidence (margin). On the other hand, the absolute value of the derivative of $\tilde{l}^\alpha$ decreases as $\alpha$ increases. This behavior makes the computation of the optimal classification function more challenging as $\alpha$ increases (as evidenced by the intractability of $0$-$1$ loss).

\subsection{Classification-Calibration}

An important concept in the analysis and design of margin-based losses is that of classification-calibration. To define this, we begin by defining the true posterior $\eta: \mathcal{X} \rightarrow [0,1]$ as $ \eta(x) = P_{Y|X}(y=1 | x)$. As in \cite{bartlett2006convexity}, we abbreviate $\eta(x)$ as $\eta$, making implicit the dependence on $x$.

\begin{definition}[{\cite[Definition~1]{bartlett2006convexity}}]
A margin-based loss function $\tilde{l}$ is said to be classification-calibrated if, for every $\eta \neq 1/2$,
\begin{equation}
\label{cc} \inf_{f:f(2\eta - 1) \leq 0}(\eta \tilde{l}(f) + (1 - \eta)\tilde{l}(-f)) > \inf_{f\in\mathbb{R}}(\eta \tilde{l}(f) + (1 - \eta)\tilde{l}(-f)).
\end{equation}
\end{definition}

The conditional risk of $f$ given $\{X=x\}$ is given by
\begin{equation}
    \mathbb{E}_{Y|X=x}[\tilde{l}(Yf(x))] = \eta(x) \tilde{l}(f(x)) + (1 - \eta(x))\tilde{l}(-f(x)).
\end{equation}
If $\tilde{l}$ is a classification-calibrated margin-based loss function, then the minimum conditional risk given $\{X=x\}$ is attained by a $z_x^*$ such that $\textnormal{sign}(z_x^*) = \textnormal{sign}(2\eta(x)-1)$. Thus, assuming that the posterior distribution $\eta$ is known, the optimal classification function for $\tilde{l}$, namely $f^*(x) := z_x^*$, gives rise to the optimal classification function for the $0$-$1$ loss, namely the Bayes decision rule $\textnormal{sign}(2\eta(x)-1)$.

The following proposition establishes another important consequence of classification-calibration; we will use it in the sequel.

\begin{prop}[{\cite[Theorem~3]{bartlett2006convexity}}] \label{prop1}
Assume that $\tilde{l}$ is a classification-calibrated margin-based loss function. Then, for every sequence of measurable functions $(f_{i})_{i=1}^\infty$ and every probability distribution on $\mathcal{X} \times \mathcal{Y}$,
\begin{equation}
\lim_{i\to\infty} R_{\tilde{l}}(f_{i}) = R_{\tilde{l}}^* \text{ implies that } \lim_{i\to\infty} R(f_{i}) = R^*,
\end{equation}
where $R_{\tilde{l}}^* := \min_f R_{\tilde{l}}(f)$ and $R^* := \min_f R(f)$.
\end{prop}
\begin{figure}[ht!]
\vspace{-.4cm}
\includegraphics[width=0.5\textwidth,trim=160 20 100 25,clip]{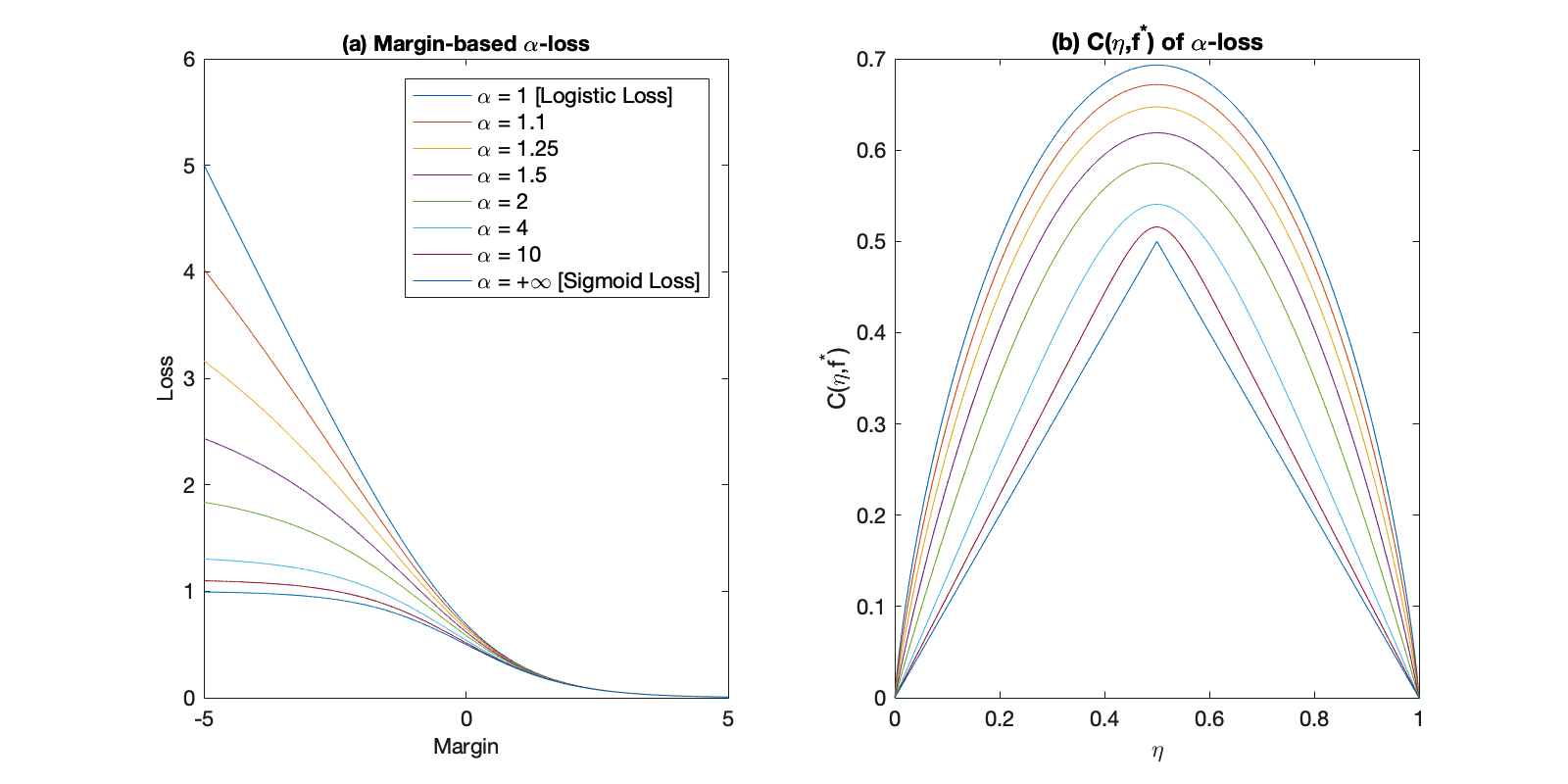}
\caption{(a) Margin-based $\alpha$-loss, as a function of the margin $z=yf(x)$; (b) minimum conditional risk for different values of $\alpha$. \label{Fig:MarginBasedLosses}}
\end{figure}

\section{Results}

\subsection{Relation Between $\alpha$-loss and its Margin Form}

The following proposition shows an important relation between $\alpha$-loss and its margin form in the context of binary classification. For reasons of brevity, we refer the reader to the full version of the paper for the complete proof.

\begin{prop} \label{prop2}
Consider a soft classifier $g$ and let $P_{\hat{Y}|X}$ be the set of beliefs associated to it. If $f(x) = \sigma^{-1}(g(x))$, then, for every $\alpha\in[1,\infty]$,
\begin{equation}
\label{eq:lglf}
    l^{\alpha}(y,P_{\hat{Y}|X=x}) = \tilde{l}^{\alpha}(yf(x)).
\end{equation}
Conversely, if $f$ is a classification function, then the set of beliefs $P_{\hat{Y}|X}$ associated to $g(x) := \sigma(f(x))$ satisfies \eqref{eq:lglf}. In particular, for every $\alpha\in[1,\infty]$,
\begin{equation}
\min_{P_{\hat{Y}|X}} \mathbb{E}_{X,Y}(l^\alpha(Y,P_{\hat{Y}|X})) = \min_f \mathbb{E}_{X,Y}(\tilde{l}^\alpha(Yf(X))).
\end{equation}
\end{prop}

{This proposition unifies the probabilistic and margin settings. It also illustrates that the choice of the sigmoid function as the "change of variable" between soft classifiers and classification functions is sensible as the values of the minimization are the same. Furthermore, the minimizers are one-to-one by construction.}



\subsection{Statistical Guarantees}

Now we establish some statistical properties of the margin-based $\alpha$-loss that guarantee its appropriateness for classification tasks.

\begin{theorem}
\label{result1}
For every $\alpha\in[1,\infty]$, the margin-based $\alpha$-loss $\tilde{l}^{\alpha}$ is classification-calibrated. In addition, its optimal classification function is given by
\begin{equation}
\label{optimalclassifier}
    f^{*}(\alpha,\eta) = \alpha \cdot \sigma^{-1}(\eta).
\end{equation}
Furthermore, its minimum conditional risk is given by
\begin{equation}
C_{\tilde{l}^{\alpha}}(\eta,f^{*}) =
\begin{cases} 
    -\eta \log{\eta} - (1-\eta) \log{1 - \eta} & \alpha = 1, \\
    \frac{\alpha}{\alpha - 1} \left[1-Q(\eta)-Q(1-\eta)\right] & \alpha \in (1,+\infty), \\
    \min\{\eta, 1- \eta \} & \alpha \rightarrow +\infty,
\end{cases}
\end{equation}
where $\displaystyle Q(z) = \left(\frac{z^{\alpha + 1 - 1/\alpha}}{z^{\alpha} + (1-z)^{\alpha}}\right)^{1 - 1/\alpha}$.
\end{theorem}

\begin{proof}
If $\alpha = 1$, then $\tilde{l}^{\alpha}$ becomes logistic loss which is classification-calibrated, as is shown in \cite{bartlett2006convexity}. Its optimal classifier and minimum conditional risk are given in \cite{masnadi2009design}. If $\alpha = +\infty$, then $\tilde{l}^{\alpha}$ becomes sigmoid loss which is known to be classification-calibrated \cite{bartlett2006convexity}. It can be verified that the optimal classifier for sigmoid loss is degenerated, i.e., \begin{equation} f^{*}(+\infty,\eta) = \begin{cases} +\infty & \eta > 1/2, \\ -\infty & \eta < 1/2, \\ \end{cases} \end{equation} and $C^{*}_{\tilde{l}^{\infty}} = \min\{\eta, 1-\eta\}.$

Let $\alpha\in(1,+\infty)$. By definition of classification-calibration, we have to show that, for every $\eta \neq 1/2$,
\begin{equation} \label{plug}
\inf_{f:f(2\eta - 1) \leq 0}(\eta \tilde{l}(f) + (1 - \eta)\tilde{l}(-f)) > \inf_{f\in\mathbb{R}}(\eta \tilde{l}(f) + (1 - \eta)\tilde{l}(-f)).
\end{equation}
First we assume that $\eta > 1/2$. In this case, the strategy of proof is to show that the optimization in the right-hand-side of \eqref{plug} has a unique minimizer $f^*$ and that $f^*>0$, which means that the right-hand-side of \eqref{plug} is strictly smaller than the left-hand-side. Indeed, with some straightforward algebra, we can show that
$
    f^{*} = \alpha \log{\Big(\dfrac{\eta}{1-\eta}\Big)},
$
which trivially implies that $f^{*} > 0$. The value of $C^{*}_{\tilde{l}^{\alpha}}$ can be obtained by substituting $f^*$ in \eqref{eq:DefConditionalRisk}. The case $\eta < 1/2$ can be proved {\it mutatis mutandis}.
\end{proof}

\begin{prop}
The margin-based $\alpha$-loss $\tilde{l}^{\alpha}:\overline{\mathbb{R}}\to\mathbb{R}_+$ is convex for $\alpha = 1$ and quasi-convex for $\alpha > 1$. Furthermore, for every $\alpha\in[1,\infty]$, the minimum conditional risk $C_{\tilde{l}^{\alpha}}(\eta,f^{*})$ is concave as a function $\eta$.
\end{prop}

\begin{proof}
Since $\tilde{l}^{1}$ is logistic loss, it is convex with respect to the margin as can be seen by observing its second derivative. For $\alpha > 1$, it can be shown that $\tilde{l}^{\alpha}$ is monotone, so it is quasi-convex. However, $\tilde{l}^{\alpha}$ is not convex for $\alpha > 1$ since its second derivative is negative for negative values of the margin. Similarly, using a second-derivative argument it can be shown that $C_{\tilde{l}^{\alpha}}(\eta,f^{*})$ is concave for every $\alpha \in [1,+\infty]$.
\end{proof}

Many commonly used loss functions in binary classification are convex. Despite the advantages of convex losses in terms of numerical optimization, non-convex loss functions can provide practical benefits as well. For instance, Mei \textit{et al.} \cite{mei2018landscape} state that non-convex loss functions ``demonstrate superior robustness and classification accuracy in contrast to convex loss functions''. In essence, non-convex loss functions assign less weight to misclassified training examples and therefore algorithms using such losses are less perturbed by outliers. The desirability of non-convex losses is further evidenced by other empirical studies, see, for example, \cite{wu2007robust,chapelle2009tighter,nguyen2013algorithms}. 

Another perspective on the convexity of loss functions is presented in \cite{masnadi2009design} where the authors argue that, for classification tasks, the convexity of a margin-based loss function is non-essential, as long as its minimum conditional risk is concave as a function of $\eta$. With regards to $\alpha$-loss, this is amply observed in Figure 1(b).
Since the margin-based $\alpha$-loss is classification-calibrated and its minimum conditional risk is concave as a function of $\eta$, it is a reasonable loss function for binary classification problems.
\subsection{Empirical Landscape of $\alpha$-loss under Logistic Regression} 

In this section we consider a setting in which logistic regression is used to perform binary classification. Namely, for a given $\Theta\subset\mathbb{R}^d$, the family of soft classifiers under consideration has the form 
\begin{equation}
g_{\theta}(x) = \sigma(\theta \cdot x),
\end{equation}
where $\theta\in\Theta$ and $\sigma$ is the sigmoid function given in \eqref{eq:DefSigmoid}. This in turn results in $\alpha$-loss taking the form
\begin{align}
\label{emp}
\nonumber l^{\alpha}(y,g_{\theta}(x)) &= \frac{\alpha}{\alpha - 1} \Big[1 - \frac{1 + y}{2}g_{\theta}(x)^{1 - 1/\alpha}\\
& \quad \quad \quad \quad \quad - \frac{1 - y}{2}(1 - g_{\theta}(x))^{1 - 1/\alpha}\Big].
\end{align} 
A straightforward computation shows that
\begin{align}
    \nonumber \frac{\partial}{\partial \theta_{i}} l^{\alpha}(y,g_{\theta}(x)) &= \Big[\frac{1-y}{2}g_\theta(x)(1-g_\theta(x))^{1-1/\alpha}\\
    \label{der} & \quad \quad \quad - \frac{1+y}{2} g_\theta(x)^{1-1/\alpha}(1-g_\theta(x))\Big]x_i,
\end{align}
where $\theta = (\theta_1,\ldots,\theta_d)$ and $x = (x_1,\ldots,x_d)$. Hence,
\begin{equation} \label{1stderiv}
\nabla_\theta l^{\alpha}(Y,g_\theta(X)) = F_{1}(\alpha,\theta,X,Y)X,
\end{equation}
where $F_{1}(\alpha,\theta,x,y)$ is the expression within brackets in \eqref{der}.

Recently, Mei \textit{et al.}~\cite{mei2018landscape} prove that for non-convex loss functions satisfying certain regularity conditions, there exists a bijection between the critical points of the empirical risk and the critical points of true risk such that the distance between corresponding points decreases at a rate $O\left(\sqrt{\log{n}/n}\right)$, where $n$ is the sample size. Building upon their work, we establish generalization bounds for logistic regression under $\alpha$-loss. 

\begin{theorem}
\label{thm:Landscape}
Let $B_d(r)$ denote the ball of radius $r$ in $d$-dimensional Euclidean space. Assume that, for some $r>0$, $X$ is supported over $B_d(r)$ and $\theta \in \Theta \subset B_d(r)$. For each $y\in\{-1,1\}$, 
let $X^{[y]}$ be a random variable having the distribution of $X$ conditioned on $Y=y$. We further assume that $X^{[1]} \stackrel{\textnormal{d}}{=} -X^{[-1]}$, $\mathbb{E}[X^{[1]}]\neq0$, and $1-\sigma(-r^2)^2 < \frac{\|\mathbb{E}(X^{[1]})\|}{\mathbb{E}(\|X^{[1]}\|)}$. Let $\hat{\theta}_{n}$ denote a local minimizer of the empirical risk function $\theta \mapsto \hat{R}_{l^{\alpha}}(g_{\theta})$. If the sample size $n$ is large enough, then, with probability at least $1-\delta$,
\begin{equation}
\label{land}
    |R_{l^{\alpha}}(g_{\hat{\theta}_{n}}) - \hat{R}_{l^{\alpha}}(g_{\hat{\theta}_{n}})| \leq C_\alpha \left(\sqrt{\frac{\log(n)}{n}} + 
    \sqrt{\frac{\log(4m/\delta)}{2n}}\right),
\end{equation}
where $C_\alpha$ is a constant independent of $n$ and $m$ is the number of critical points.
\end{theorem}
\begin{proof}[Proof]
In Appendices~\ref{Appendix:BackgroundMai} and \ref{Appendix:AssumptionsMai} we show that $l^\alpha$ satisfies the regularity conditions\footnote{These conditions are sub-Gaussian gradient, sub-exponential Hessian, Lipschitz Hessian, and strongly Morse expected risk.} in \cite[Thm.~2]{mei2018landscape} and, as a result, the expected risk has finitely many critical points $\{\theta_1,\ldots,\theta_m\}$ and for $n$ large enough, with probability at least $1-\delta/2$, there exists $\hat{\theta} := \theta_{i}$ for some $i \in [m]$ such that,
\begin{equation}
\label{eq:DistanceCriticalPoints}
    \|\hat{\theta}_{n} - \hat{\theta}\| \leq C\sqrt{\frac{\log(n)}{n}},
\end{equation}
where $C$ is a constant independent of $n$. By the triangle inequality, 
\begin{equation}
|R_{l^{\alpha}}(g_{\hat{\theta}_{n}}) - \hat{R}_{l^{\alpha}}(g_{\hat{\theta}_{n}})| \leq {\rm I} + {\rm II} + {\rm III},
\end{equation}
where ${\rm I} = |R_{l^{\alpha}}(g_{\hat{\theta}_{n}}) - R_{l^{\alpha}}(g_{\hat{\theta}})|$,  ${\rm II} = |R_{l^{\alpha}}(g_{\hat{\theta}}) - \hat{R}_{l^{\alpha}}(g_{\hat{\theta}})|$, and ${\rm III} = |\hat{R}_{l^{\alpha}}(g_{\hat{\theta}}) - \hat{R}_{l^{\alpha}}(g_{\hat{\theta}_{n}})|$.

Observe that,
$
{\rm II} \leq \max_{i=1,\ldots,m} |R_{l^{\alpha}}(g_{\theta_i}) - \hat{R}_{l^{\alpha}}(g_{\theta_i})|.
$
By Hoeffding's inequality and the union bound, see, e.g., \cite[Chapter~4]{shalev2014understanding}, it can be shown that, for any $\epsilon>0$,
\begin{align}
\nonumber &\Pr\left(\max_{i=1,\ldots,m} |R_{l^{\alpha}}(g_{\theta_i}) - \hat{R}_{l^{\alpha}}(g_{\theta_i})| > \epsilon\right)\\
& \quad \quad \quad \quad \quad \quad \quad \quad \quad \leq 2m\exp\left(-\frac{2n(\alpha-1)^2\epsilon^{2}}{\alpha^2}\right). \label{eq:LargeDeviationMax}
\end{align}
By taking $\delta = 4m\exp\left(-2n(\alpha-1)^2\epsilon^{2}/\alpha^2\right)$, we conclude that, with probability at least $1-\delta/2$,
\begin{equation}
\label{eq:BoundII}
    {\rm II} \leq \max_i |R_{l^{\alpha}}(g_{\theta_i}) - \hat{R}_{l^{\alpha}}(g_{\theta_i})| \leq \frac{\alpha}{\alpha-1}\sqrt{\frac{\log(4m/\delta)}{2n}}.
\end{equation}

By the boundedness of $X$ and $\theta$, the derivative in \eqref{1stderiv} is bounded for all $X$ and $\theta$. Therefore, independently of the training dataset, the empirical risk function $\hat{R}_{l^{\alpha}}$ is $C_{\alpha}''$-Lipschitz for some $C_{\alpha}'' \geq 0$. Hence,
\begin{equation}
{\rm III} \leq C_{\alpha}'' \|\hat{\theta}_{n} - \hat{\theta}\|.
\end{equation}
The last inequality and \eqref{eq:DistanceCriticalPoints} imply that
\begin{equation}
\label{eq:BoundIII}
    {\rm III} \leq C_\alpha' \sqrt{\frac{\log(n)}{n}},
\end{equation}
where $C_\alpha' := C C_\alpha''$.

A differentiation under the integral sign argument shows that $R_{l^{\alpha}}$ is also $C_\alpha''$-Lipschitz. Thus,
\begin{equation}
|R_{l^{\alpha}}(g_{\hat{\theta}_{n}}) - R_{l^{\alpha}}(g_{\hat{\theta}})| \leq C_\alpha'' \|\theta_{n} - \theta\|.
\end{equation}
As before, \eqref{eq:DistanceCriticalPoints} leads to
\begin{equation}
\label{eq:BoundI}
{\rm I} = |R_{l^{\alpha}}(g_{\hat{\theta}_{n}}) - R_{l^{\alpha}}(g_{\hat{\theta}})| \leq C_\alpha' \sqrt{\frac{\log(n)}{n}}.
\end{equation}
The result follows from \eqref{eq:BoundII}, \eqref{eq:BoundIII} and \eqref{eq:BoundI}.
\end{proof}

The following corollary follows as a natural addendum to our main results and establishes that an algorithm perfectly trained using the $\alpha$-loss converges, with the number of samples $n$, to an optimal hypothesis w.r.t.~the $0$-$1$ loss.


\begin{corollary} \label{cor1}
For each $n\in\mathbb{N}$, let $S_{n}$ be a training dataset of size $n$ and $\hat{\theta}_n$ be a global minimizer of the associated empirical risk function $\theta \mapsto \hat{R}_{l^{\alpha}}(g_{\theta})$. Under the assumptions of Theorem~\ref{thm:Landscape}, the sequence $(\hat{\theta}_n)_{n=1}^\infty$ is asymptotically optimal for the $0$-$1$ risk, i.e., almost surely,
\begin{equation}
\lim_{n\to\infty} R(\hat{\theta}_n) = R^*.
\end{equation}
\end{corollary} The Proof of Corollary 1 is given in Appendix F.

\subsection{Simulation Results}
We perform simulations on a logistic regression model with randomly initialized weights using a portion of the MNIST dataset. In order to have a binary dataset, we partition the MNIST dataset into the images of $1$'s and $7$'s which yields a training set of $12,500$ samples and a test set of $2,050$ samples (evenly divided between the two labels for both train and test data). Of the $12,500$ training samples, we use $11,500$ for training and the remaining $1,000$ for cross-validation.

Since cross entropy (log-loss, i.e. $\alpha = 1$) is the most commonly used loss function for practical implementation in classification \cite{janocha2017loss}, we use it as our benchmark for accuracy. In this way, we compare cross entropy and $\alpha$-loss in terms of accuracy for $\alpha \in \{1.1,1.2,1.5,2.0\}$. In order to have a level playing field, we tune the learning rate during cross-validation, so as to compare the optimal performance of each loss function. 

\begin{table}[h!]
\centering
\vspace{.4cm}
\begin{tabular}{||c c c||} 
 \hline
 $\alpha$ & Learning Rate & Testing Accuracy \\ [1ex] 
 \hline\hline
 1.0 & 1.0 & $\textbf{85.3805\%}$ \\ 
 1.1 & 1.3 & $85.4005\%$ \\
 1.2 & 1.0 & $85.8527\%$\\
 1.5 & 1.9 & $87.3044\%$\\
 2.0 & 2.0 & $\textbf{87.3302\%}$\\
 \hline
\end{tabular}
\caption{Performance Regression}
\label{table:1}
\end{table}

As shown in Table \ref{table:1}, for the simple logistic regression model under consideration, $\alpha$-loss with $\alpha=2$ exhibits a testing accuracy about $\sim2\%$ higher than cross entropy. While this is a simple model, the performance of $\alpha$-loss is encouraging and suggests that further work is needed.

It ought to be mentioned that, with large-capacity models, MNIST data can be classified with an accuracy above 99\% \cite{MNIST}. The goal of our numerical experiments with low capacity models (such models are desirable when tuning deep neural networks is challenging) is to show that $\alpha$-loss can perform better than cross entropy in some situations. Further simulations using state-of-the-art datasets is the subject of ongoing research.

\section{Concluding Remarks}
We have proved theoretical properties and highlighted practical preliminary results for $\alpha$-loss under binary classification. Beyond generalization to multi-hypothesis testing, the optimal choice of $\alpha$ is another important problem and will require exploring the trade-off between the magnitude of the gradients (convergence) and the gradient noise induced by finite samples. Yet another challenging problem to explore is the robustness of $\alpha$-loss for $\alpha>1$ against adversarial examples; one approach to doing so is by quantifying its generalization properties by building upon the work in \cite{xu2017information}.
\bibliographystyle{IEEEtran}
\bibliography{citations}

\begin{thebibliography}{10}
\providecommand{\url}[1]{#1}
\csname url@samestyle\endcsname
\providecommand{\newblock}{\relax}
\providecommand{\bibinfo}[2]{#2}
\providecommand{\BIBentrySTDinterwordspacing}{\spaceskip=0pt\relax}
\providecommand{\BIBentryALTinterwordstretchfactor}{4}
\providecommand{\BIBentryALTinterwordspacing}{\spaceskip=\fontdimen2\font plus
\BIBentryALTinterwordstretchfactor\fontdimen3\font minus
  \fontdimen4\font\relax}
\providecommand{\BIBforeignlanguage}[2]{{%
\expandafter\ifx\csname l@#1\endcsname\relax
\typeout{** WARNING: IEEEtran.bst: No hyphenation pattern has been}%
\typeout{** loaded for the language `#1'. Using the pattern for}%
\typeout{** the default language instead.}%
\else
\language=\csname l@#1\endcsname
\fi
#2}}
\providecommand{\BIBdecl}{\relax}
\BIBdecl

\bibitem{lin2004note}
Y.~Lin, ``A note on margin-based loss functions in classification,''
  \emph{Statistical \& Probability Letters}, vol.~68, no.~1, pp. 73--82, 2004.

\bibitem{bartlett2006convexity}
P.~L. Bartlett, M.~I. Jordan, and J.~D. McAuliffe, ``Convexity, classification,
  and risk bounds,'' \emph{Journal of the American Statistical Association},
  vol. 101, no. 473, pp. 138--156, 2006.

\bibitem{nguyen2009}
X.~Nguyen, M.~J. Wainwright, and M.~I. Jordan, ``On surrogate loss functions
  and $f$-divergences,'' \emph{AOS}, vol.~37, no.~2, pp. 876--904, 04 2009.

\bibitem{masnadi2009design}
H.~Masnadi-Shirazi and N.~Vasconcelos, ``On the design of loss functions for
  classification: theory, robustness to outliers, and {S}avage{B}oost,'' in
  \emph{Advances in neural information processing systems}, 2009, pp.
  1049--1056.

\bibitem{mei2018landscape}
S.~Mei, Y.~Bai, A.~Montanari \emph{et~al.}, ``The landscape of empirical risk
  for nonconvex losses,'' \emph{The Annals of Statistics}, vol.~46, no.~6A, pp.
  2747--2774, 2018.

\bibitem{liao2018tunable}
J.~Liao, O.~Kosut, L.~Sankar, and F.~P. Calmon, ``A tunable measure for
  information leakage,'' in \emph{2018 IEEE International Symposium on
  Information Theory (ISIT)}.\hskip 1em plus 0.5em minus 0.4em\relax IEEE,
  2018, pp. 701--705.

\bibitem{janocha2017loss}
K.~Janocha and W.~M. Czarnecki, ``On loss functions for deep neural networks in
  classification,'' \emph{arXiv preprint arXiv:1702.05659}, 2017.

\bibitem{wu2007robust}
Y.~Wu and Y.~Liu, ``Robust truncated hinge loss support vector machines,''
  \emph{Journal of the American Statistical Association}, vol. 102, no. 479,
  pp. 974--983, 2007.

\bibitem{chapelle2009tighter}
O.~Chapelle, C.~B. Do, C.~H. Teo, Q.~V. Le, and A.~J. Smola, ``Tighter bounds
  for structured estimation,'' in \emph{Advances in neural information
  processing systems}, 2009, pp. 281--288.

\bibitem{nguyen2013algorithms}
T.~Nguyen and S.~Sanner, ``Algorithms for direct 0--1 loss optimization in
  binary classification,'' in \emph{International Conference on Machine
  Learning}, 2013, pp. 1085--1093.

\bibitem{shalev2014understanding}
S.~Shalev-Shwartz and S.~Ben-David, \emph{Understanding machine learning: From
  theory to algorithms}.\hskip 1em plus 0.5em minus 0.4em\relax Cambridge
  University Press, 2014.

\bibitem{MNIST}
Y.~LeCun, C.~Cortes, and C.~J.~C. Burges, ``The {MNIST} database of handwritten
  digits,'' \url{http://yann.lecun.com/exdb/mnist/index.html}.

\bibitem{xu2017information}
A.~Xu and M.~Raginsky, ``Information-theoretic analysis of generalization
  capability of learning algorithms,'' in \emph{Advances in Neural Information
  Processing Systems}, 2017, pp. 2524--2533.

\bibitem{boyd2004convex}
S.~Boyd and L.~Vandenberghe, \emph{Convex optimization}.\hskip 1em plus 0.5em
  minus 0.4em\relax Cambridge University Press, 2004.

\bibitem{vershynin2018high}
R.~Vershynin, \emph{High-dimensional probability: An introduction with
  applications in data science}.\hskip 1em plus 0.5em minus 0.4em\relax
  Cambridge University Press, 2018.

\end{thebibliography}
\newpage
\section{Appendix}
\subsection{Proof of Proposition 2}
Consider a soft classifier $g$ and let $P_{\hat{Y}|X}$ be the set of beliefs associated to it. Suppose $f(x) = \sigma^{-1}(g(x))$, where $g(x) = P_{\hat{Y}|X}(1|x)$. We want to show that 
\begin{equation} \label{desiredeq} l^{\alpha}(y,P_{\hat{Y}|X=x}) = \tilde{l}^{\alpha}(yf(x)).\end{equation} 
We assume that $\alpha \in (1,\infty)$. Note that the cases where $\alpha = 1$ and $\alpha = \infty$ follow similarly. 
\par
Suppose that $g(x) = P_{\hat{Y}|X}(1|x) = \sigma(f(x))$. If $y = 1$, then \begin{align} l^{\alpha}(1,P_{\hat{Y}|X}(1|x)) &= l^{\alpha}(1,\sigma(f(x))) \\ &= \frac{\alpha}{\alpha - 1}[1 - \sigma(f(x))^{1 - 1/\alpha}] \\ 
&= \tilde{l}^{\alpha}(f(x)).\end{align} 
If $y = -1$, then \begin{align} l^{\alpha}(-1,P_{\hat{Y}|X}(-1|x)) &= l^{\alpha}(-1,1 - P_{\hat{Y}|X}(1|x))
\\ &= l^{\alpha}(-1,1 - \sigma(f(x))) \\ \label{step} &= l^{\alpha}(-1,\sigma(-f(x))) \\ &= \frac{\alpha}{\alpha - 1}[1 - \sigma(-f(x))^{1 - 1/\alpha}]\\ &= \tilde{l}^{\alpha}(-f(x)),\end{align} where \eqref{step} follows from
\begin{equation}\label{simprop} \sigma(x) + \sigma(-x) = 1, \end{equation}
which can be observed by \eqref{eq:DefSigmoid}.
To show the reverse direction of \eqref{desiredeq} we substitute \begin{equation} f(x) = \sigma^{-1}(g(x)) = \sigma^{-1}(P_{\hat{Y}|X}(1|x))\end{equation}
in $\tilde{l}^{\alpha}(yf(x))$. 
For $y = 1$, \begin{align} \tilde{l}^{\alpha}(f(x)) &= \tilde{l}^{\alpha}(\sigma^{-1}(P_{\hat{Y}|X}(1|x))) 
\\ &= \frac{\alpha}{\alpha - 1}[1 - (\sigma(\sigma^{-1}(P_{\hat{Y}|X}(1|x))))^{1 - 1/\alpha}]
\\&=\frac{\alpha}{\alpha - 1}[1 - P_{\hat{Y}|X}(1|x)^{1 - 1/\alpha}]
\\&= l^{\alpha}(1,P_{\hat{Y}|X}(1|x)).\end{align} 
For $y = -1$, \begin{align} \tilde{l}^{\alpha}(-f(x)) &= \tilde{l}^{\alpha}(-\sigma^{-1}(P_{\hat{Y}|X}(1|x))) 
\\ &= \frac{\alpha}{\alpha - 1}[1 - \sigma(-\sigma^{-1}(P_{\hat{Y}|X}(1|x)))^{1 - 1/\alpha}]
\\ \label{step2} &= \frac{\alpha}{\alpha - 1}[1 - (1-\sigma(\sigma^{-1}(P_{\hat{Y}|X}(1|x))))^{1 - 1/\alpha}]
\\ &= \frac{\alpha}{\alpha - 1}[1 - P_{\hat{Y}|X}(-1|x)^{1 - 1/\alpha}]
\\&= l^{\alpha}(-1, P_{\hat{Y}|X}(-1|x)),\end{align} where \eqref{step2} follows from \eqref{simprop}.
\par
The equality in the results of the minimization procedures follows from the equality between $l^{\alpha}$ and $\tilde{l}^{\alpha}$. As was shown in \cite{liao2018tunable}, the minimizer of the left-hand-side is 
\begin{equation} 
P^{*}_{\hat{Y}|X}(y|x) = \dfrac{P_{Y|X}(y|x)^{\alpha}}{\sum\limits_{y} P_{Y|X} (y|x)^{\alpha}}.
\end{equation} 
Using $f(x) = \sigma^{-1}(P_{\hat{Y}|X}(1|x))$, $f^{*}(x) = \sigma^{-1}(P^{*}_{\hat{Y}|X}(1|x))$.

\subsection{Proof of Theorem 1}
Suppose $\alpha = 1$, then $\tilde{l}^{\alpha}$ becomes 
\begin{equation} \tilde{l}^{1}(z) = -\log{(\sigma(z))} = \log{(1+e^{-z})},\end{equation} which is logistic loss. By solving the minimization procedure in \eqref{cc}, it can be shown as in \cite{bartlett2006convexity} that $\tilde{l}^{1}$ is classification-calibrated. Further, the optimal classifier and minimum conditional risk of logistic loss are given in \cite{masnadi2009design}. 

Suppose $\alpha = +\infty$, then $\tilde{l}^{\alpha}$ becomes 
\begin{equation}
\tilde{l}^{\infty}(z) = 1 - \sigma(z) = \dfrac{e^{z}}{1+e^{z}},
\end{equation}
which is sigmoid loss. Similarly, sigmoid loss can be shown to be classification-calibrated as is given in \cite{bartlett2006convexity}. It can be verified by calculating the minimization procedure in \eqref{cc} that the optimal classifier for sigmoid loss is degenerate. That is, \begin{equation} f^{*}(+\infty,\eta) = \begin{cases} +\infty & \eta > 1/2 \\ -\infty & \eta < 1/2. \\ \end{cases} \end{equation} Therefore, $C_{\tilde{l}^{\infty}}(\eta,f^{*}) = \min\{\eta, 1-\eta\}.$ Note that sigmoid loss and $0$-$1$ loss have the same minimum conditional risk. Thus, sigmoid loss can be viewed as a smoothed version of $0$-$1$ loss and will similarly suffer from vanishing gradients for most values of the margin. 
\par Now consider $\alpha \in (1, +\infty)$. Since classification calibration requires proving \eqref{cc}, we begin by expanding the inequality in \eqref{cc}  
using $\tilde{\ell}$ in \eqref{mar} to show that $\forall \eta \neq 1/2$,
\begin{equation} \label{plug2}
\inf_{f:f(2\eta - 1) \leq 0}(\eta \tilde{l}(f) + (1 - \eta)\tilde{l}(-f)) > \inf_{f\in\mathbb{R}}(\eta \tilde{l}(f) + (1 - \eta)\tilde{l}(-f)).
\end{equation}
Without loss of generality, we assume that $\eta > 1/2$. The strategy of the proof is to demonstrate that for $\eta>1/2$, $f^{*}>0$, which means that the right-hand-side of \eqref{plug2} is smaller than the left-hand-side because the attainer of the infimum is not in the search-space of the left-side's infimum.
We rearrange the right-hand-side of \eqref{plug2} to obtain
\begin{equation} \label{minrisk}
\dfrac{\alpha}{\alpha - 1} \bigg[1 - \sup\limits_{f\in\mathbb{R}} \Big[\eta\Big(\dfrac{1}{1 + e^{-f}}\Big)^{1 - 1/\alpha} + (1 - \eta)\Big(\dfrac{1}{1 + e^{f}}\Big)^{1 - 1/\alpha}\Big]\bigg].
\end{equation}
We take the derivative of the expression inside the supremum, which we denote $g(\eta,\alpha,f)$, and obtain
\begin{equation} \label{derg}
\begin{split}
\dfrac{d}{df} g(\eta,\alpha,f) = \Big(1 - \dfrac{1}{\alpha}\Big)\Big(\dfrac{1}{e^{f} + 2 + e^{-f}}\Big)\Big[\eta\Big({1+e^{-f}}\Big)^{1/\alpha} \\ - (1-\eta)\Big({1+e^{f}}\Big)^{1/\alpha}\Big].
\end{split}
\end{equation}
One can then obtain the $f_0$ minimizing \eqref{minrisk} by setting $\dfrac{d}{df} g(\eta,\alpha,f) = 0$, i.e., \begin{equation} \label{deriv}
\eta\Big({1+e^{-f_{0}}}\Big)^{1/\alpha} = (1-\eta)\Big({1+e^{f_{0}}}\Big)^{1/\alpha}.
\end{equation} 
Note that the derivative $dg(\eta,\alpha,f)/df$ in \eqref{derg} approaches zero for both $f \rightarrow +\infty$ and $f \rightarrow -\infty$ for which $g$ simplifies to $\eta$ and $(1-\eta)$, respectively. Since $\eta>1/2$, to show that $f_{0} \in (-\infty,\infty)$ is the point at which $g(\eta,\alpha,f)$ is maximized, we must demonstrate that $g(\eta,\alpha,f_{0}) > \eta > 1/2$. We solve \eqref{deriv} for $(1-\eta)$ and substitute it into $g(\eta,\alpha,f_{0})$. Further simplifying, we obtain $\eta (1+e^{-f_{0}})^{1/\alpha}$ which is always greater than $\eta$. Therefore, $f_{0}$ is the maximizer of $g(\eta,\alpha,f)$. Solving \eqref{deriv} for $f_{0}$, we obtain 
\begin{equation} 
\label{attainer} 
f_{0} = f^{*}(\alpha,\eta) = \alpha \log{\Big(\dfrac{\eta}{1-\eta}\Big)}>0, \end{equation} 
i.e., $\tilde{l}^{\alpha}$ is classification-calibrated. Since \eqref{attainer} minimizes the right side of \eqref{plug2}, it is the optimal classifier for $\tilde{l}^{\alpha}$ where $\alpha \in (1,+\infty)$. Accordingly, $C_{\tilde{l}^{\alpha}}(\eta,f^{*})$ is obtained by substituting \eqref{attainer} into \eqref{minrisk}.

\subsection{Proof of Proposition 3}
For $\alpha = 1$, $\tilde{l}^{1}(z) = -\log{\sigma(z)}$. Further, \begin{equation} 
\dfrac{d^{2}}{dz^{2}} \tilde{l}^{1}(z) = \dfrac{e^{-z}}{(1+e^{-z})^{2}} \geq 0,
\end{equation}
$\forall z \in \mathbb{R}$, so $\tilde{l}^{1}$ is convex.
\par
For $\alpha \in (1,\infty)$, 
\begin{equation} \label{2ndderiv}
\dfrac{d^2}{dz^{2}} \tilde{l}^{\alpha}(z) = \dfrac{(e^{-z}+1)^{1/\alpha}e^{z}(\alpha e^z - \alpha + 1)}{\alpha(e^{z} + 1)^{3}}.
\end{equation}
As can be observed in the numerator for $\alpha > 1$, there exists some $z_{0}$ for which $\alpha e^{z_{0}} - \alpha + 1 < 0$. Thus $\tilde{l}^{\alpha}$ is not convex for $\alpha \in (1,\infty)$. 
Similarly as can be seen in \eqref{2ndderiv} by letting $\alpha \rightarrow \infty$, that $\dfrac{d^2}{dz^{2}} \tilde{l}^{\infty}(z) = \dfrac{e^z(e^z - 1)}{(e^z + 1)^3}$, which is less than zero for $z<0$. Thus, $\tilde{l}^{\infty}$ is also not convex. 
\par
It can be shown that, for all $ \alpha \in [1,\infty]$,  $\tilde{l}^{\alpha}$ is monotonically decreasing since 
\begin{equation}
\dfrac{d}{dz}\tilde{l}^{\alpha}(z) = \dfrac{-(e^{-z} + 1)^{1/\alpha}e^{z}}{(1+e^{z})^{2}} < 0,
\end{equation}
$\forall z \in \mathbb{R}.$
Since monotonic functions are quasi-convex \cite{boyd2004convex}, we have that $\tilde{l}^{\alpha}$ is quasi-convex for $\alpha > 1$.
\par
With regards to the minimum conditional risk, for $\alpha = 1$, it can be shown that $\dfrac{d^2}{d\eta^2} C_{\tilde{l}^{1}}(\eta,f^{*}) = \dfrac{1}{(\eta-1)\eta} < 0$ since $\eta \in (0,1)$. Despite a cumbersome expression, one can similarly verify that, for $\alpha \in (1,\infty)$, $C_{\tilde{l}^{\tilde{\alpha}}}(\eta,f^{*})$ is concave. For $\alpha = \infty$, $C_{\tilde{l}^{\tilde{\infty}}}(\eta,f^{*}) = \min\{\eta,1-\eta\}$ can be easily verified to be concave as a function of $\eta$.

\subsection{Background for Theorem~\ref{thm:Landscape}}
\label{Appendix:BackgroundMai}

The proof of Theorem 2 relies on a result by Mei {\it et al.}~\cite{mei2018landscape} stated at the end of this section. We start by providing the necessary background. 



\begin{definition} A random vector $X\in \mathbb{R}^{d}$ is $\sigma^{2}$-sub-Gaussian if, for every $\lambda\in \mathbb{R}^{d}$,
\begin{equation}
\mathbb{E}[e^{\langle \lambda,X - \mathbb{E}[X] \rangle}] \leq e^{\sigma^{2} \|\lambda\|^{2}_{2}/2},
\end{equation}
where $\langle \cdot,\cdot \rangle$ denotes the inner product. 
\end{definition}

Gaussian and bounded random variables are examples of sub-Gaussian random variables, see, for example, \cite{vershynin2018high}. It can be shown that if the components of a random vector are sub-Gaussian, then the random vector itself is sub-Gaussian \cite{vershynin2018high}.


\begin{definition}
A random matrix $Z$ is $\tau^{2}$-sub-exponential if, for every $\lambda \in B_{d}(1/\tau)$, 
\begin{equation}
\mathbb{E}\left[e^{|Z_\lambda - \mathbb{E}[Z_\lambda]|}\right] \leq 2,
\end{equation}
where $Z_\lambda := \langle \lambda,Z\lambda\rangle$ and $B_d(r)$ denotes the ball of radius $r$ in $d$-dimensional Euclidean space.
\end{definition}




We now recall the definition of a regularity property known as strongly Morse. Let $[d] := \{1,2,\ldots,d\}$.

\begin{definition}
We say that a twice differentiable function $F: B_{d}(r) \rightarrow \mathbb{R}$ is $(\epsilon,\eta)$-strongly Morse if $\|\nabla F(x) \|_{2} > \epsilon$ for $\|x\|_{2} = r$ and, for any $x \in \mathbb{R}^{d}$, $\|x\|_{2} < r$, the following holds: 
\begin{equation}
\label{eq:DefMorse}
\|\nabla F(x)\|_{2} \leq \epsilon \implies \min\limits_{i \in [d]} |\lambda_{i}(\nabla^{2}F(x))| \geq \eta,
\end{equation}
where $\{\lambda_i(\nabla^{2}F(x)): i\in[d]\}$ are the eigenvalues of $\nabla^{2}F(x)$.
\end{definition}

Now we are in position to state Mei {\it et al.}~result.

\begin{prop}[{\cite[Thm.~2]{mei2018landscape}}] \label{montanari}
Let $l$ be a given loss function. Assume that
\begin{itemize}
    \item[1)] the gradient $\nabla_{\theta} l(\theta)$ is sub-Gaussian;
    \item[2)] the Hessian $\nabla^{2}_{\theta} l(\theta)$ is sub-exponential;
    \item[3)] the Hessian $\nabla_{\theta}^{2} R_{l}(\theta)$ is bounded at a point and Lipschitz continuous with integrable Lipschitz constant, i.e, there exists $J_{*}$ such that
    \begin{equation}
    \label{eq:DefJz}
    J(\mathbf{z}) \equiv \sup\limits_{\theta_{1} \neq \theta_{2} \in B^{p}(r)} \dfrac{\|\nabla^{2}l(\theta_{1};\mathbf{z}) - \nabla^{2}l(\theta_{2};\mathbf{z})\|_{op}}{\|\theta_{1} - \theta_{2}\|_{2}},
    \end{equation} 
    where $\mathbb{E}[J(\mathbf{Z})] \leq J_{*}$;
    \item[4)] $R_{l}(\theta)$ is $(\epsilon,\eta)$-strongly Morse.
\end{itemize}
Let $\hat{\theta}_{n}$ denote a local minimizer of the empirical risk function $\theta \mapsto \hat{R}_{l}(\theta)$. If the sample size $n$ is large enough, then there exists a critical point $\hat{\theta}$ of the true risk function $\theta \mapsto R_{l}(\theta)$ such that, with probability at least $1-\delta$,
\begin{equation}
\|\hat{\theta}_{n} - \hat{\theta}\|_{2} \leq C\sqrt{\frac{\log{n}}{n}},
\end{equation}
where $C = C(\sigma,\alpha,\epsilon,\eta,d)$ is a positive constant. Further, $\hat{R}_{l}(\theta)$ is $(\epsilon/2,\eta/2)$-strongly Morse.
\end{prop}

\subsection{Proof that $l^{\alpha}$ satisfies the Assumptions of Proposition 4}
\label{Appendix:AssumptionsMai}

Here, we prove the assumptions stipulated by Proposition~\ref{montanari} hold for $\alpha$-loss. We restrict ourselves to the setting of logistic regression. Thus, $\hat{R}_{l^{\alpha}}(\theta) = \hat{R}_{l^{\alpha}}(g_{\theta})$ and $R_{l^{\alpha}}(\theta) = R_{l^{\alpha}}(g_{\theta})$, where $g_{\theta}(x) = \sigma(\theta \cdot x)$ and $g_{\theta}(x)$ is often abbreviated $g_{\theta}$ for convenience. 

\subsubsection*{Proof of Assumption 1}

The first assumption requires the gradient of the loss function to be sub-Gaussian. The gradient of $\alpha$-loss is given by \eqref{1stderiv}. That is, 
\begin{equation} \label{1stderivative}
\nabla_\theta l^{\alpha}(Y,g_\theta(X)) = F_{1}(\alpha,\theta,X,Y)X,
\end{equation}
where 
\begin{align}
\nonumber F_1(\alpha,\theta,x,y) &= \frac{1-y}{2}g_\theta(x)(1-g_\theta(x))^{1-1/\alpha}\\
\label{F1} & \quad \quad - \frac{1+y}{2} g_\theta(x)^{1-1/\alpha}(1-g_\theta(x)).
\end{align}
In order to prove \eqref{F1}, we used that
\begin{equation} \label{sigprop}
\dfrac{\partial}{\partial \theta} \sigma(\theta \cdot x) = \sigma(\theta \cdot x)(1-\sigma(\theta \cdot x)).
\end{equation}
By the boundedness of the sigmoid function, we have that $|F_{1}(\alpha,\theta,X,Y)| \leq 1$. Since $X\in B_d(r)$ by assumption, each component of $\nabla_{\theta} l^{\alpha}$ is bounded and, as a consequence, sub-Gaussian. Therefore, the gradient of $\alpha$-loss is sub-Gaussian.

\subsubsection*{Proof of Assumption 2}

\begin{figure*}[ht]
\begin{equation}
\label{eq:DefF2}
F_2(\alpha,\theta,x,y) = \dfrac{1 - y}{2}\Big[g_{\theta} (1 - g_{\theta})^{2 - 1/\alpha} - \left(1 - \dfrac{1}{\alpha}\right)g_{\theta}^{2} (1 - g_{\theta})^{1 - 1/\alpha}\Big] +\dfrac{1+y}{2} \Big[g_{\theta}^{2 - 1/\alpha}(1- g_{\theta}) - \left(1 - \dfrac{1}{\alpha}\right)g_{\theta}^{1 - 1/\alpha}(1 - g_{\theta})^{2} \Big]
\end{equation}
\end{figure*}

The second assumption requires the Hessian of the loss function to be sub-exponential. It can be shown that the Hessian has the form
\begin{equation} \label{nab2}
\nabla^{2}_{\theta}l^{\alpha}(Y,g_{\theta}(X)) = F_{2}(\alpha,\theta,X,Y)XX^{T},
\end{equation}
where $F_2(\alpha,\theta,X,Y)$ is defined on \eqref{eq:DefF2}. It is straightforward to verify that $|F_{2}(\alpha,\theta,X,Y)|\leq \frac{1}{4}$. Notice that the product of $\nabla^{2}_{\theta}l^{\alpha}$ with $\lambda \in B^{p}(1)$ becomes 
\begin{equation} \label{summation}
\langle \lambda,\nabla^{2}_{\theta}l^{\alpha}\lambda\rangle = \left(F_2(\alpha,\theta,X,Y)^{1/2} \sum_{i=1}^d \lambda_i X_i\right)^2.
\end{equation}
Since both $\theta$ and $X$ are assumed to be bounded, $\langle \lambda,\nabla^{2}_{\theta}l^{\alpha}\lambda\rangle$ is the square of a bounded random variable. Since the square of a sub-gaussian random variable is sub-exponential, we conclude that the Hessian is sub-exponential.

\subsubsection*{Proof of Assumption 3}

\begin{figure*}[hb]
\begin{equation}
\begin{split}
F_{3}(\alpha,\theta,x,y) &= \dfrac{1-y}{2}\Big[g_{\theta}(1-g_{\theta})^{3-\frac{1}{\alpha}} 
- \left(4+\dfrac{1}{\alpha}\right)g_{\theta}^{2}(1-g_{\theta})^{2-\frac{1}{\alpha}} + \left(1-\dfrac{1}{\alpha}\right)^{2}g_{\theta}^{3}(1-g_{\theta})^{1-\frac{1}{\alpha}}\Big]\\ \label{eq:DefF3}
& \quad \quad \quad -\dfrac{1+y}{2}\Big[g_{\theta}^{3 - \frac{1}{\alpha}}(1-g_{\theta})
-\left(4 + \dfrac{1}{\alpha}\right)(1-g_{\theta})^{2}g_{\theta}^{2-\frac{1}{\alpha}}
+ \left(1-\dfrac{1}{\alpha}\right)^{2}(1-g_{\theta})^{3}g_{\theta}^{1-\frac{1}{\alpha}}\Big]%
\end{split}
\end{equation}
\end{figure*}

The third required assumption is that the Hessian of the loss function is Lipschitz and the Hessian of the population risk is bounded above at a point. The former can be observed by calculating the third derivative of $\alpha$-loss and showing that it is bounded. The third derivative has the form
\begin{equation}
\dfrac{\partial}{\partial \theta_{i}} \nabla^{2}_{\theta}l^{\alpha}(Y,g_{\theta}(X)) = F_{3}(\alpha,\theta,X,Y)XX^{T}X_{i}, 
\end{equation}
where $F_{3}(\alpha,\theta,X,Y)$ is defined in \eqref{eq:DefF3}.
Observe that $|F_{3}(\alpha,\theta,X,Y)|\leq 2$. 
Since $\theta,X\in B_d(r)$ by assumption, the derivative of the Hessian is bounded with constant $L = 2r^{3}$. Therefore, the Hessian is Lipschtiz continuous, in the sense of \eqref{eq:DefJz}, with integrable Lipschitz constant $L$. Using similar arguments, it is straightforward to verify that the Hessian of the population risk is bounded at a point.


\subsubsection*{Proof of Assumption 4}
The final assumption requires the population risk to be strongly Morse.
Recall that, for each $y\in\{-1,1\}$, $X^{[y]}$ has the same distribution as $X$ conditioned on $Y=y$. Since $X^{[1]} \stackrel{\textnormal{d}}{=} -X^{[-1]}$ by assumption, conditioning on $Y$ we obtain that
\begin{equation}
    \nabla_\theta R(\theta) = - \mathbb{E}\left[g_\theta(X^{[1]})^{1-1/\alpha}g_\theta(-X^{[1]})X^{[1]}\right].
\end{equation}
Observe that
\begin{align}
    & \|\mathbb{E}[g_\theta(X^{[1]})^{1-1/\alpha}g_\theta(-X^{[1]})X^{[1]}] - \mathbb{E}[X^{[1]}]\|\\
    &= \|\mathbb{E}[(g_\theta(X^{[1]})^{1-1/\alpha}g_\theta(-X^{[1]}) - 1) X^{[1]}]\|\\
    &\leq \mathbb{E}[|g_\theta(X^{[1]})^{1-1/\alpha}g_\theta(-X^{[1]}) - 1| \|X^{[1]}\|],
\end{align}
where we used the convexity of the norm and Jensen's inequality. Since $\theta,X\in B^d(r)$, it can be verified that
\begin{equation}
    \sigma(-r^2)^2 \leq g_\theta(X^{[1]})^{1-1/\alpha}g_\theta(-X^{[1]}) \leq 1.
\end{equation}
Hence,
\begin{align}
    & \|\mathbb{E}[g_\theta(X^{[1]})^{1-1/\alpha}g_\theta(-X^{[1]})X^{[1]}] - \mathbb{E}[X^{[1]}]\|\\
    &\leq (1-\sigma(-r^2)^2) \mathbb{E}[\|X^{[1]}\|].
\end{align}
By the triangle inequality, we obtain that
\begin{equation}
    \sigma(-r^2)^2 \mathbb{E}[\|X^{[1]}\|] \leq \|\nabla_\theta R(\theta)\|.
\end{equation}
By assumption, $\|\mathbb{E}(X)\| \neq 0$, hence $R(\theta) > \epsilon$ for all $\theta$, where
\begin{equation}
    \epsilon:= \sigma(-r^2)^2 \mathbb{E}[\|X^{[1]}\|].
\end{equation}
Therefore, by vacuity, $R(\theta)$ satisfies \eqref{eq:DefMorse} for every $\eta>0$, i.e., $R(\theta)$ is $(\epsilon,\eta)$-strongly Morse.

\subsection{Proof of Corollary 1}

We start by proving that, almost surely,
\begin{equation}
\label{eq:CorollaryNotMargin}
   \lim_{n\to\infty} R_{l^{\alpha}}(g_{\hat{\theta}_{n}}) = \min_{\theta\in\Theta} R_{l^{\alpha}}(g_\theta).
\end{equation}
Let $\theta^*$ be a minimizer of the expected risk, i.e.,
\begin{equation}
R_{l^\alpha}(g_{\theta^*}) = \min_{\theta\in\Theta} R_{l^\alpha}(g_\theta).
\end{equation}
Observe that
\begin{equation}
\label{eq:SplitInII}
    0 \leq R_{l^\alpha}(g_{\hat{\theta}_n}) - R_{l^\alpha}(g_{\theta^*}) = {\rm I}_n + {\rm II}_n,
\end{equation}
where ${\rm I}_n := R_{l^\alpha}(g_{\hat{\theta}_n}) - \hat{R}_{l^\alpha}(g_{\hat{\theta}_n})$ and ${\rm II}_n := \hat{R}_{l^\alpha}(g_{\hat{\theta}_n}) - R_{l^\alpha}(g_{\theta^*})$. After some straightforward manipulations, \eqref{land} implies that, for every $\epsilon>0$,
\begin{equation}
\label{eq:GeneralizationCriticalPoints}
    \mathbb{P}\left(|R_{l^\alpha}(g_{\hat{\theta}_n}) - \hat{R}_{l^\alpha}(g_{\hat{\theta}_n})| > \epsilon\right) < 4mn e^{-n\epsilon^2/(2C_\alpha^2)},
\end{equation}
whenever $n$ is large enough. A routine application of the Borel-Cantelli lemma shows that, almost surely,
\begin{equation}
\label{eq:asLimit1}
    \lim_{n\to\infty} {\rm I}_n = \lim_{n\to\infty} R_{l^\alpha}(g_{\hat{\theta}_n}) - \hat{R}_{l^\alpha}(g_{\hat{\theta}_n}) = 0.
\end{equation}
Since $\hat{\theta}_n$ is a minimizer of the empirical risk $\hat{R}_{l^\alpha}$,
\begin{equation}
    {\rm II}_n = \hat{R}_{l^\alpha}(g_{\hat{\theta}_n}) - R_{l^\alpha}(g_{\theta^*}) \leq \hat{R}_{l^\alpha}(g_{\theta^*}) - R_{l^\alpha}(g_{\theta^*}).
\end{equation}
By Hoeffding's inequality, for every $\epsilon>0$,
\begin{equation}
\label{eq:HoeffdingEmpiricalExpected}
    \mathbb{P}\left(|\hat{R}_{l^\alpha}(g_{\theta^*}) - R_{l^\alpha}(g_{\theta^*})| > \epsilon \right) \leq 2e^{-2n(\alpha-1)^2\epsilon^2/\alpha^2}.
\end{equation}
Hence, the Borel-Cantelli lemma implies that, almost surely,
\begin{equation}
    \lim_{n\to\infty} |\hat{R}_{l^\alpha}(g_{\theta^*}) - R_{l^\alpha}(g_{\theta^*})| = 0.
\end{equation}
In particular, we have that, almost surely,
\begin{equation}
\label{eq:asLimit2}
    \limsup_{n\to\infty} {\rm II}_n \leq 0.
\end{equation}
By plugging \eqref{eq:asLimit1} and \eqref{eq:asLimit2} in \eqref{eq:SplitInII}, we obtain that, almost surely,
\begin{equation}
    0 \leq \limsup_{n\to\infty} \left[R_{l^\alpha}(g_{\hat{\theta}_n}) - R_{l^\alpha}(g_{\theta^*})\right] \leq 0,
\end{equation}
from which \eqref{eq:CorollaryNotMargin} follows.

For each $n\in\mathbb{N}$, let $f_n:\mathcal{X}\to\overline{\mathbb{R}}$ be given by $f_{n}(x) = \hat{\theta}_{n} \cdot x$. Since $f_{n}(x) = \sigma^{-1}(\sigma(\hat{\theta}_{n}\cdot x)) = \sigma^{-1}(g_{\hat{\theta}_{n}}(x))$, Proposition~\ref{prop2} and \eqref{eq:CorollaryNotMargin} imply that
\begin{equation}
\label{eq:CorollaryMargin}
    \lim_{n\to\infty} R_{\tilde{l}^{\alpha}}(f_{\hat{\theta}_{n}}) = \min_{\theta\in\Theta} R_{\tilde{l}^{\alpha}}(f_{\theta}) =: R_{\tilde{l}^{\alpha}}^*.
\end{equation}
Since $\tilde{l}^{\alpha}$ is classification-calibrated, as established in Theorem~\ref{result1}, Proposition~\ref{prop1} and \eqref{eq:CorollaryMargin} imply that
\begin{equation}
\lim_{n\to\infty} R(\hat{\theta}_n) = R^*,
\end{equation}
as required.

\end{document}